\documentclass{article}

\usepackage{arxiv}

\usepackage[utf8]{inputenc} 
\usepackage[T1]{fontenc}    
\usepackage{hyperref}       
\usepackage{url}            
\usepackage{booktabs}       
\usepackage{amsfonts}       
\usepackage{nicefrac}       
\usepackage{microtype}      
\usepackage{lipsum}
\usepackage{graphicx}
\usepackage{natbib}
\usepackage{amsmath}
\usepackage{amsthm}
\usepackage{float} 
\usepackage{xcolor}
\usepackage{bm}
\usepackage{caption}
\usepackage{subcaption}
\usepackage{color,soul}
\graphicspath{ {./images/} }
\newcommand{\R}{\mathbb{R}}
\newcommand{\spa}{\text{span}}

\newtheorem{lemma}{Lemma}

\title{Change-of-Basis Pruning via Rotational Invariance}

\author{
  Alex Ning\thanks{Equal contribution} \\
  Department of Computer Science\\ 
  University of Virginia\\
  \texttt{rnx2bc@virginia.edu} \\
   \And
  Vainateya Rangaraju\footnotemark[1] \\
  Department of Computer Science\\
  University of Virginia\\
  \texttt{prr3gw@virginia.edu} \\
}

\begin{document}
\maketitle
\begin{abstract}
Structured pruning removes entire neurons or channels, but its effectiveness depends on how importance is distributed across the representation space. Change-of-basis (CoB) pruning addresses this challenge by applying orthogonal linear transformations that concentrate importance within certain dimensions. However, many standard deep learning architectures are not inherently invariant to such transformations. To enable compatibility, we introduce two-subspace radial activations (TSRAs): an activation family that is invariant to orthogonal linear transformations applied independently within its two activation subspaces. This invariance allows CoB transformations to be merged into surrounding weights without incurring extra parameters. We position this work as a proof-of-concept that a rotationally invariant design may offer a principled approach towards change-of-basis pruning. We do not provide an analysis of multiple TSRA candidates nor do we explore weight initialization for any TSRAs. These limitations, combined with other necessary modifications we make to permit rotational invariance, result in a slight accuracy drop of $4.52\%$ compared to a ReLU-based control. However, using activation-magnitude importance, VGG-16 implementing our CoB+TSRA framework shows encouraging results on CIFAR-10. Under fixed-ratio structured pruning, CoB improves accuracy over a TSRA baseline at all pruning ratios and extends reliable pruning frontier from roughly $30\%$ to $70\%$ of parameters \textit{without} post-prune fine tuning. Under threshold-based pruning strategies, CoB prunes $90-96\%$ of parameters while maintaining $1-6\%$ accuracy drop after fine-tuning. Together, these results indicate that rotationally invariant architectures may offer a promising path towards CoB pruning.
\end{abstract}

\section{Background}

\subsection{Model Pruning}

Modern machine learning (ML) models often contain millions of parameters and require substantial computational resources \citep{He_2024}. Model compression addresses this challenge by reducing model size, computational load, and memory usage. Among compression techniques, model pruning has become particularly prominent in recent literature \citep{He_2024}. Pruning removes unnecessary parameters, neurons, or even whole dimensions from a neural network. The goal is to produce a smaller, more efficient model while preserving accuracy and overall performance as much as possible \citep{He_2024}.

Broadly, pruning methods fall into two categories. In unstructured pruning, weights with the smallest absolute values are set to zero \citep{He_2024}. By contrast, structured pruning removes larger, semantically meaningful units (i.e entire neurons, channels, filters, or dimensions)~\citep{He_2024}. A central question in pruning is how to determine what to remove. Classical approaches define importance scores/metrics either directly on the model parameters or on the activations that those parameters produce. Weight-based importance methods rely on statistics such as the L1 or L2 norm of the weight corresponding to each neuron or channel \citep{li2017pruningfiltersefficientconvnets}. Activation-based importance instead looks at the behavior of the network on data, ranking parameters or groups of parameters by importance metrics such as the magnitude or variance of their activations across a set of data samples \citep{li2017pruningfiltersefficientconvnets} \citep{molchanov2019importanceestimationneuralnetwork}. These scoring strategies are then used either layerwise \citep{li2017pruningfiltersefficientconvnets}, where each layer prunes a fixed proportion of its units, or globally \citep{He_2024}, where all units across the network compete in a single ranked list.

In this context, our proposed approach can be viewed as a structured, activation magnitude-based pruning method, with one key distinction: rather than applying importance metrics in the network’s native coordinate system, we compute an orthogonal linear transformation (a "change-of-basis" or "rotation"\footnote{We note that strictly speaking, \textit{half} of all orthogonal linear transformations are \textit{not} rotations since they have a determinant = $-1$. Throughout this paper, we commit a slight abuse of terminology by using "rotation" to refer to \textit{all} orthogonal linear transformations. Although not strictly accurate, we find this perspective sometimes assists intuition.}) that concentrates importance along certain component axes. This is enabled via modifications to the neural network hidden layer which make it invariant to the applied transformations. Pruning is then performed in this rotated basis, enabling cleaner and more principled removal of entire dimensions.

\begin{figure}[H]
    \centering
    \setlength{\fboxsep}{0pt}
    \includegraphics[width=\textwidth]{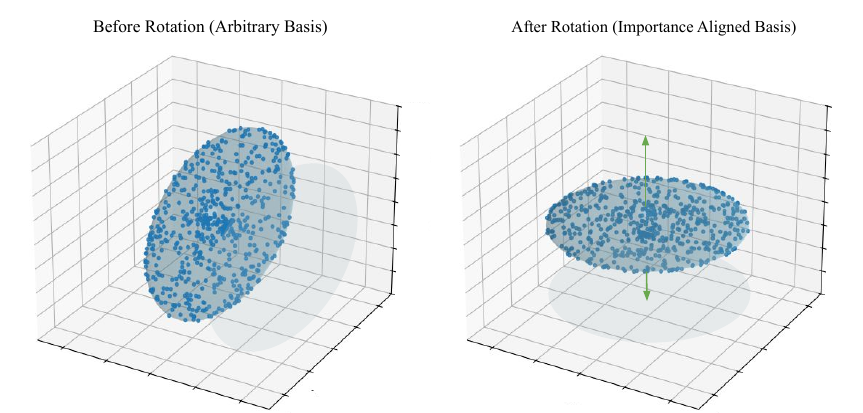}
    \caption{Diagram conceptually demonstrating the proposed benefit of change-of-basis pruning, which respects hidden state geometry. On the left: model activations depicted in the original basis. Here neurons/original basis components contribute roughly equally and it might be quite unclear what to prune. On the right: same model activations depicted under a change-of-basis. Here the feature space is reordered such that dimension 1 and 2 captures the most variance, and dimension 3 very little. Pruning is clearly most feasible along the green axis. 
  }
    \label{fig:intuition_cob_pruning}
\end{figure}

\subsection{Related Works}

Recent work in the pruning of Transformer-based~\citep{vaswani2023attentionneed} language models have explored inserting orthogonal linear transformations into the transformer architecture to allow more favorable pruning. By merging some of the inserted linear transformations into existing parameters where possible, these methods concentrate an importance metric into \textit{some} parameters while reducing them in others, enabling a greater proportion of parameters to be pruned. SliceGPT~\citep{ashkboos2024slicegptcompresslargelanguage} utilizes a PCA-based approach to find the change-of-basis transformation which concentrates a L2-norm metric within the outputs of each transformer layer and then applies structured pruning. In contrast, RotPruner~\citep{chen2025rotpruner} and DenoiseRotator~\citep{gu2025denoiserotatorenhancepruningrobustness} use different importance metrics and learned rotations which can be applied towards unstructured, semi-structured, and/or structured pruning. While these methods have achieved substantive success in pruning and have demonstrated the viability of pruning in a rotated space, a limitation also exists in that not all of the inserted linear transformations can be merged into neighboring parameters, resulting in some amount of additional parameter count and compute usage.

\cite{ganev2023universalapproximationmodelcompression} introduced radial neural networks, which utilize non-element-wise radial rescaling activations. These activations rescale a given input vector dependent only on the norm of that vector. That is, given a function $f: \R \to \R$, a vector $x \in \R^d$, and $|\cdot |$ denoting the Euclidean norm, then a radial rescaling activation $\phi$ would be of form $\phi(x) = f(|x|)x$. \cite{ganev2023universalapproximationmodelcompression} proved universal approximation theorem for bounded-width radial neural networks. Additionally, they demonstrate a property of \textit{lossless compression} for radial neural networks and find an algorithm which, given a radial neural network, can find an \textit{exactly} equivalent smaller network.

\section{Two-Subspace Radial Activations}
\label{sec:tsras}

The property of lossless compression of radial neural networks permits an algorithmic method for reducing model size~\citep{ganev2023universalapproximationmodelcompression}. Equivalently, hidden layers of radial neural networks \textit{saturate} at a maximum width, above which additional dimensions do not add any expressivity. This is unlike in element-wise activation functions, where additional layer width increases representational capacity with no theoretical bound. As a result, although radial neural networks \textit{with} bounded width do satisfy a universal approximation theorem, the inability to rapidly increase model complexity via increased hidden layer dimensionality may result in practical limitations.

We begin by formally characterizing the saturated width limitations of radial neural networks.

Let $d_{i-1}$ be the input size of the $i$-th hidden layer of a radial neural network and let $d_i$ be its output size. Denote $W \in \R^{d_i \times d_{i-1}}$ its weight matrix and $b \in \R^{d_i}$ its optional bias. Given scalar function $f: \R \to \R$, let $\phi(x) = f(|x|)x$ be the radial rescaling activation for $x \in \R^{d_i}$.

Let $x_\text{in} \in \R^{d_{i-1}}$ be the input to the $i$-th hidden layer. If a bias is not used, let $x_\text{pre} = Wx_\text{in} \in \R^{d_i}$ be the pre-activation of the hidden layer. Otherwise, denote $x_\text{pre} = Wx_\text{in} + b$ as the pre-activation.

Although $x_\text{pre} \in \R^{d_i}$, the image of $W$ spans a subspace of at most $d_{i-1}$ dimensions. Thus, if no bias is used, $x_\text{pre}$ is restricted to at maximum a $d_{i-1}$-dimensional subspace of the $d_i$-dimensional space: $\dim(\spa (\{ Wx_\text{in} : x_\text{in} \in \R^{d_{i-1}}\})) \leq d_{i-1}$. If a bias is used, then the span of all possible pre-activations occupies a subspace at most one dimension larger: $\dim(\spa (\{ Wx_\text{in} : x_\text{in} \in \R^{d_{i-1}}\} \cup \{b\})) \leq d_{i-1} + 1$.

Denote $x_\text{act} = \phi(x_\text{pre})$ be the activation of the $i$-th hidden layer. Because $\phi$ only \textit{rescales} the vector $x_\text{pre}$, $x_\text{act}$ still sits on the line spanned by $x_\text{pre}$. As a result, given $S$ as the span of all possible pre-activations $x_\text{pre}$, $\phi(S) \subseteq S$, meaning that the activations of the hidden layer are restricted to the same subspace (the "occupied" subspace) of $\R^{d_i}$ that the pre-activations are. This enables the \textit{lossless compression} property of radial neural networks, as the rotational invariance of the radial rescaling activations permits a change-of-basis to be performed on $W$ which aligns the occupied subspace with some of the output coordinate axes. The matrix $W$ can then be pruned in a lossless manner by removing the coordinate axes orthogonal to the occupied subspace. Eg. the pruned matrix $W' \in \R^{d_{i-1} \times d_{i-1}}$ in the case of no bias and $W' \in \R^{d_{i-1} + 1 \times d_{i-1}}$ in the case of bias. As a result, without a bias, the capacity of radial neural networks does not benefit from \textit{any} width expansion beyond the width of the input layer. With a bias, radial neural networks can only expand width at a rate of $1$ dimension every hidden layer without exceeding width saturation.

Modern deep neural networks often take advantage of far more dramatic expansions in hidden layer width in order to increase representational capacity. As a result, we introduce \textit{two-subspace} radial activations (TSRAs), which maintain complete rotational invariance within each of two \textit{separate} subspaces while not being bound to the same restrictive width saturation limitations of radial neural networks.

Given an orthogonal decomposition for $\R^d = U \oplus V$, a vector $x \in \R^d$, the decomposition of $x$ into $x_U \in U$ and $x_V \in V$ such that $x = x_U + x_V$, and functions $f_U, f_V: \R^2 \to \R$, a TSRA takes the form:
\begin{equation}\label{eq:tsra}
    \sigma(x) = f_U(|x_U|, |x_V|)x_U + f_V(|x_U|, |x_V|)x_V
\end{equation}

Essentially, the TSRA rescales two separate subspaces of a vector $x$ using two separate scaling functions. Each function takes as input only the two norms of the subspace components of the vector. As a result, rotational invariance is maintained \textit{within} each subspace independently.

We prove in Appendix~\ref{sec:tsra_width_saturation_proof} that TSRAs have strictly greater width saturation than radial rescaling functions, and bound their width saturation to $2(d_{i-1} + 1)$ if a bias is used and $2d_{i-1}$ otherwise. This indicates that TSRAs \textit{may} be capable of exponential width saturation in deep neural network architectures, which enables their practical application in many modern architectures.

\section{Change-of-Basis Importance Concentration}

\begin{figure}[H]
    \centering
    \setlength{\fboxsep}{0pt}
    \includegraphics[width=\textwidth]{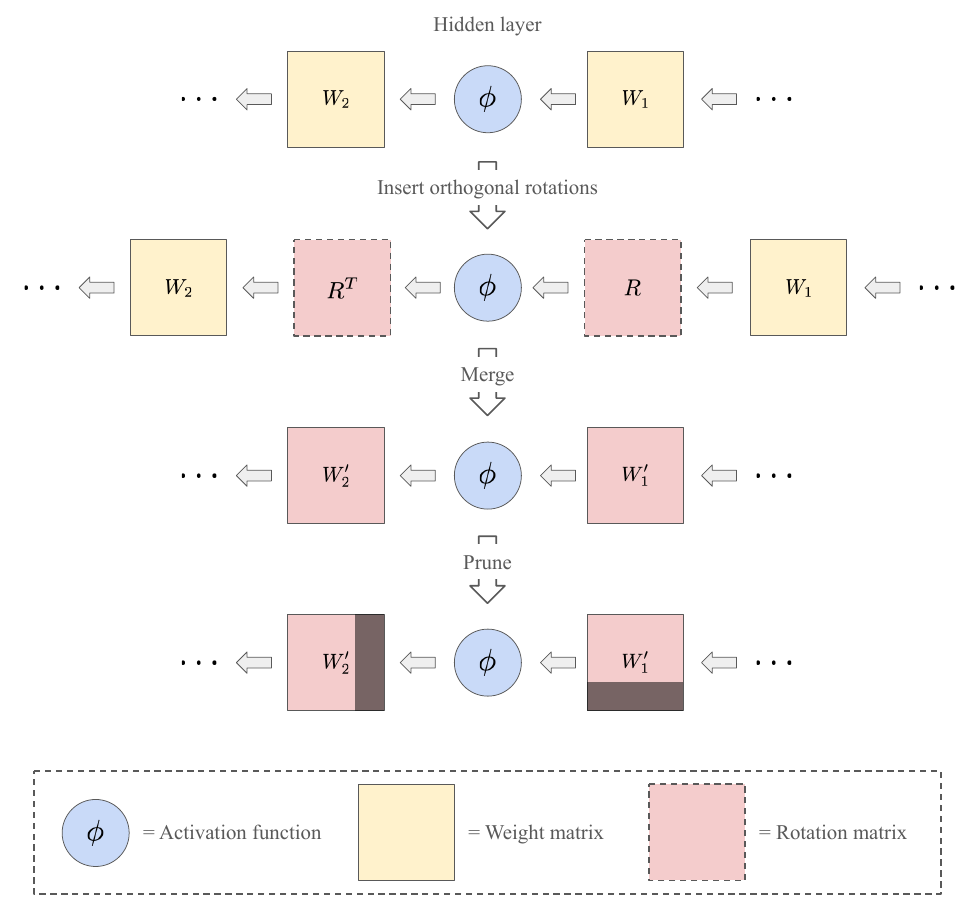}
    \caption{Diagram showing the insertion and merging of orthogonal linear transformations into the hidden layer of a fully connected neural network in order to concentrate importance into certain component axes. This is followed by pruning, which is represented by structured pruning of neurons in the diagram. Optional bias not shown for simplicity. Assumes that the activation function $\phi$ is invariant under the orthogonal linear transformation $R$.}
    \label{fig:cob_pruning}
\end{figure}

We introduce here the idea of inserting change-of-basis linear transformations into a neural network in order to concentrate importance on certain parameters and reduce importance in others. Figure~\ref{fig:cob_pruning} shows a diagram demonstrating this concept followed by pruning. While this idea is similar to the ones employed by SliceGPT~\citep{ashkboos2024slicegptcompresslargelanguage}, RotPruner~\citep{chen2025rotpruner}, and DenoiseRotator~\citep{gu2025denoiserotatorenhancepruningrobustness}, we discuss here their implementation in fully-connected neural networks (or convolutional neural networks) in a manner which does not leave the model with any additional parameters, whereas the aforementioned works were implemented in Transformer models and did leave the model with additional parameters.

Let $d_i$ be the output size of the $i$-th hidden layer of a fully connected neural network. Let $d_{i-1}$ be the output size of the previous layer and $d_{i+1}$ the output size of the following layer. Let $W_i \in \R^{d_i\times d_{i-1}}$ be the weight of the $i$-th hidden layer with optional bias $b_i \in \R^{d_i}$, and let $W_{i+1} \in \R^{d_{i+1}\times d_i}$ be the weight of the following layer with optional bias $b_{i+1} \in \R^{d_{i+1}}$.

The general approach to pruning begins by calculating importance scores for the parameters candidate for pruning. For example, in unstructured pruning this may be individual parameters, whereas in structured pruning this would be entire rows of $W_i$ and the corresponding columns of $W_{i+1}$ corresponding to individual neurons of hidden layer $i$. Then, the parameters or groups of parameters with the lowest importance scores would be pruned. Through a change-of-basis of the model weights, we aim to concentrate the importance scores onto certain parameters or parameter groups while reducing the importance of others, which should enable a greater amount of parameters or parameter groups to be pruned.

However, without rotational invariance of the activation function, this approach is not straightforward. Let $\phi$ be the activation function for hidden layer $i$ and let $x_\text{in} \in \R^{d_{i-1}}$ be the input to the hidden layer. Then we can express the forward pass on the $i$-th hidden layer as:

$$\text{hidden\_layer}_i(x_\text{in}) = \phi(W_ix_\text{in} + b_i)$$

Let $R \in \text{SO}(d_i)$ be an orthonormal rotation matrix and $R^T$ its inverse. Since $RR^T = I$, we can introduce the orthonormal rotation to the forward pass:

$$\text{hidden\_layer}_i(x_\text{in}) = \phi(\bm{R^TR}W_ix_\text{in} + \bm{R^TR}b_i)$$

We can merge $R$ into $W_i$ and $b_i$ to create $W_i' = RW_i$ and $b_i' = Rb_i$:

$$\text{hidden\_layer}_i(x_\text{in}) = \phi(R^T\bm{W_i'}x_\text{in} + R^T\bm{b_i'}) = \phi(R^T(W_i'x_\text{in} + b_i'))$$

This is equivalent to performing a change-of-basis on $W_i$ and $b_i$. Assuming $R$ is formulated to concentrate importance, we could now perform more aggressive pruning on $W_i'$ and $b_i'$. However, if $\phi$ is \textit{not} invariant to rotation, then we cannot remove $R^T$, and permanently \textit{gain} the additional computation, memory usage, and parameters associated with $R^T$. Given our objective of pruning, this is likely not desirable.

If $\phi$ is invariant to rotation however, then the following equivalency holds:

$$\text{hidden\_layer}_i(x_\text{in}) = \phi(R^T(W_i'x_\text{in} + b_i')) = R^T\phi(W_i'x_\text{in} + b_i')$$

Which allows us to factor $R^T$ into the weight matrix of the \textit{next} hidden layer:

\begin{align*}
    \text{hidden\_layer}_{i+1}&(x_\text{in})\\
    &=\phi(W_{i+1}\text{hidden\_layer}_i(x_\text{in}) + b_{i+1})\\
    &=\phi(\bm{W_{i+1}R^T}\phi(W_i'x_\text{in} + b_i') + b_{i+1})\\
    &=\phi(\bm{W_{i+1}'}\phi(W_i'x_\text{in} + b_i') + b_{i+1})\\
\end{align*}

Where $W_{i+1}' = W_{i+1}R^T$. This approach allows us to implement change-of-basis importance concentration with no additional compute or parameters required.

Of course, if $\phi$ was invariant to rotation, then it would be a radial rescaling activation (see Section~\ref{sec:tsras}). A TSRA is only invariant to rotation in its two separate subspaces $U$ and $V$. As such, with TSRAs, we must restrict $R$ to only those orthogonal linear transformations which perform \textit{separate} orthogonal linear transformations within the subspaces $U$ and $V$ while leaving both invariant.

\section{Methodology}

We make our code and pretrained model weights available at\\ \url{https://github.com/apning/change-of-basis-pruning}.

\subsection{Change-of-Basis Importance Concentration with PCA for Structured Pruning}

We note here that the concept discussed in this section is essentially the same as the one utilized by SliceGPT~\citep{ashkboos2024slicegptcompresslargelanguage}. Principal component analysis (PCA) is a commonly used linear dimensionality reduction technique which performs a linear transformation on a set of high dimensional data points $X \in \R^{d \times n}$ where $d$ is the dimensionality and $n$ the sample count by aligning the component axes with the \textit{principal components} that explain the most variance in descending order and then preserving the top $k \leq d$ components. That is, the \textit{first} axis will be aligned with the principal component of \textit{most} variance, the second axis with the principal component of second greatest variance, and so on until the $k$-th axis. Typically, the steps for performing PCA to reduce $X$ from $d$ to $k$ dimensions are as follows:

\begin{enumerate}
    \item Mean center $X$ to create $\tilde{X} = X - \bar{X}$.
    \item Compute the covariance matrix of the mean centered data $\Sigma = \frac{1}{n}\tilde{X}\tilde{X}^T \in \R^{d \times d}$.
    \item Find the eigenvectors and eigenvalues of $\Sigma$. Each eigenvector is a principal component and its corresponding eigenvalue represents the amount of variance explained by that principal component.
    \item The eigenvectors will be sorted by descending eigenvalue. Then, the top $k$ eigenvectors will be kept as the top $k$ principal components
    \item A linear transformation is constructed from the retained eigenvectors. This $k \times d$ matrix can then be used to project the data $X$ onto the first $k$ principal components.
\end{enumerate}

If $k = d$, then the linear transformation is orthogonal and simply performs a change-of-basis to align with the principal components ordered by maximal variance. However, if the data $X$ is \textit{not} mean-centered in step $1$, then the principal components will not maximize variance but rather the sum of squared projections along each principal component~\citep{jolliffe2016pca}. This is equivalent to maximizing the L2 norm of the projections along each principal axis, since maximizing the squared L2 norm is equivalent to maximizing the L2 norm (as the square root is a monotonic function). As a result, when using the L2-norm of each component axis of activations as our importance metric for structured pruning, we use this un-centered PCA approach to compute the change-of-basis transformation which directly maximizes the importance metric along components in descending order.

\subsection{Experiments}

Across the literature, the workflow for structured pruning is consistent at a high level. A dense model is first trained to convergence. Importance scores are then calculated for each neuron, filter, or dimension using a defined metric (e.g, L2 norm, variance, etc.). These units are then ranked and a pruning target is applied. Finally, the pruned model is fine-tuned to recover any lost accuracy. We adopt this recipe as well, building off influential pruning papers such as Li et al. (2017) \citep{li2017pruningfiltersefficientconvnets}, Liu et al. (2017) \citep{liu2017learningefficientconvolutionalnetworks}, Molchanov et al. (2019) \citep{molchanov2019importanceestimationneuralnetwork}, and He et al. (2017) \citep{he2017channelpruningacceleratingdeep}, which serve as a foundation for many recent works.

In this section, we describe the pruning benchmarks used to evaluate the effectiveness of our change-of-basis (CoB) method. Our goal is to compare standard structured pruning against CoB-augmented pruning under identical conditions. To isolate the effect of CoB, all pruning schedules rely on the same importance metric: activation-magnitude, while varying only the pruning schedule itself (e.g., fixed-ratio layerwise or threshold-based) as we discuss below. In order to calculate the importance score, we perform a forward pass on the model using random samples selected from the training set and capture the activations at each hidden layer. Then, the importance score for each component (dimension) is calculated via the L2-norm across the sampled activations.

We employ the VGG-16~\citep{simonyan2015deepconvolutionalnetworkslargescale} CNN model and CIFAR-10~\citep{Krizhevsky09learningmultiple} dataset commonly used for related pruning works. Because CIFAR-10 consists of $32 \times 32$ images, far smaller than the $224 \times 224$ images that VGG-16 was originally designed for, we greatly reduce the input features of the fully-connected classifier of VGG-16 from $25,088$ to $512$, which reduces initial parameter count from $134.3$M parameters to just $33.6$M parameters.

For our control, we employ VGG-16 using BatchNorm~\citep{ioffe2015batchnormalizationacceleratingdeep}, the ReLU (rectified linear unit) activation function, and Kaiming uniform weight initialization~\citep{he2015delvingdeeprectifierssurpassing}. For our CoB experiments, in order to ensure the essential property of rotational invariance within hidden layers as required by the change-of-basis operation, we replace the max-pooling operation in VGG16 with average-pooling since max-pooling does not satisfy rotational invariance. For the same reason, we utilize unlearned RMSNorm~\citep{zhang2019rootmeansquarelayer} instead of the more standard BatchNorm. Although a change-of-basis can be performed on the learned gain parameter for RMSNorm $g \in \R^d$, the resulting parameter $g' = QgQ^T$ would in general be a $d \times d$ square \textit{matrix} instead of a size $d$ column vector, potentially doubling the parameter count of the hidden layer. We have not investigated weight initialization for our TSRA activation function and use the Pytorch~\citep{paszke2019pytorchimperativestylehighperformance} default initializations.

Let $\R^d = U \oplus V$ be an orthogonal decomposition of the activation space into non-zero subspaces $U$ and $V$. Also, let $x = x_U + x_V$ be the decomposition of $x$ where $x_U \in U$ and $x_V \in V$. In creating an effective TSRA, we consider a construction which is both smooth and scale-free. Towards a scale-free construction, we use only the ratio $r = \frac{|x_U|}{|x|}$. Towards smoothness, we utilize the logistic function $\lambda(r)$ with coefficients $a, b \in \R$:

$$\lambda(r) = \frac{1}{1+e^{-a(r-b)}}$$

Using the TSRA formulation in Equation~\ref{eq:tsra}, we then set $f_U$ and $f_V$ as $\lambda$ with separate coefficients $a_U, b_U$ for $f_U$ and $a_V, b_V$ for $f_V$. For these results, we utilize $a_U, a_V = 5$, $b_U = 0.5$, and $b_V = 0.7$. Although arbitrary, we find that this combination yields reasonable results. As our objective is to explore change-of-basis pruning as a concept and not TSRAs specifically, we have not performed a thorough analysis of optimal values; it is very likely that there exist better coefficients.

We find that our ReLU-based control model trains better when using the SGD optimizer with Nesterov momentum~\citep{pmlr-v28-sutskever13}, while our TSRA model trains better when using the AdamW optimizer~\citep{loshchilov2019decoupledweightdecayregularization}. After training, for all pruning experiments involving TSRAs, we prune \textit{separately} within the two TSRA subspaces. We do this because the two subspaces may naturally develop different baseline importance levels even if both carrying equally significant representations, which could result in highly uneven pruning between the subspaces if the pruning were not performed separately. We prune both convolution and fully-connected layers.

\subsection{Layerwise Fixed-Ratio Pruning}
\label{subsec:layerwise_fixed_ratio_pruning}

Our first experimental setting follows the classical pruning paradigm introduced by Li et al. (2017) \citep{li2017pruningfiltersefficientconvnets}, in which each layer independently removes a fixed proportion of its hidden dimensions. This schedule serves as a standard baseline in structured pruning research because it provides a simple, architecture-agnostic way to control pruning percentage while preserving the relative width of each layer.

In this experiment, we apply a uniform pruning ratio to every prunable layer in the network, ranging from $10\%$ to $95\%$. For each layer, activation-magnitude scores are computed, ranked, and the lowest-scoring dimensions are removed. The resulting pruned model is then evaluated both before and after fine-tuning.

To assess the impact of change-of-basis rotations, we run each pruning percentage setting twice: once using standard layerwise pruning, and once where CoB is applied prior to importance-score computation. This pairing allows us to directly measure whether CoB yields more compression-tolerant representations under an identical pruning schedule and compare them both to a pretrained baseline as well. Overall, this experiment provides a controlled and recognized benchmark for quantifying the structural benefits introduced by CoB.

\subsection{Layerwise Threshold-Based Pruning}

Our second experimental setting evaluates structured pruning under threshold-based criteria, a family of methods that remove activation dimensions whose importance scores fall below a certain cutoff rather than enforcing a fixed per-layer pruning ratio. Thresholding methods are widely used in practice \citep{molchanov2019importanceestimationneuralnetwork} because they adapt to the statistical structure of each layer: layers with many low-importance units can be pruned more aggressively, while more ``information-dense'' layers naturally retain a larger fraction of their dimensions. This adaptivity makes thresholding a natural complement to the uniform-ratio schedule explored in layerwise fixed-ratio pruning.

For each rule, the threshold $T$ determines the aggressiveness of pruning; higher thresholds eliminate more units. Because these methods do not directly prescribe the final prune ratio, they can produce widely varying percentages of parameters pruned depending on the distribution of importance scores. This property makes them particularly informative for understanding how CoB reshapes importance statistics inside the network. We evaluate threshold-based pruning \emph{only} on CoB scenarios. Broadly, we study four thresholding rules:
\begin{enumerate}
    \item \textbf{Z-score cutoff}: Removes dimensions whose importance is more than $T$ standard deviations below the mean.
    \item \textbf{Proportion-of-average}. Prunes all dimensions $d$ where $\mathrm{importance}(d) < T \cdot \mathrm{avg}$
    \item \textbf{Proportion-of-median}: Prunes all dimensions $d$ where $\mathrm{importance}(d) < T \cdot \mathrm{median}$
    \item \textbf{Proportion-of-maximum}: Prunes all dimensions $d$ where $\mathrm{importance}(d) < T \cdot \mathrm{max}$
\end{enumerate}

Together, these experiments aim to characterize how CoB interacts with data-driven pruning rules, revealing which statistics are most aligned with CoB's geometry-preserving rotations, and identify regimes where CoB enables extreme compression with minimal performance hindrance.

\begin{table}[H]
\centering
\begin{tabular}{c c c c c c c}
\toprule
\textbf{Prune \%} &
\textbf{Params After} &
\textbf{ReLU-Base} &
\textbf{TSRA–Base} &
\textbf{TSRA–CoB} &
\textbf{$\Delta$ \begin{tabular}{c}
\end{tabular}} \\
\midrule
0 & 33.6M  & 94.49 & 89.97 & 89.97 & +0.00 \\
10 & 30.3M & 92.92 & 89.69 & 89.91 & +0.22\\
20 & 26.9M & 89.01 & 88.42 & 89.78 & +1.36\\
30 & 23.7M & 82.46 & 85.44 & 89.64 & +4.20\\
40 & 20.3M & 70.14 & 76.04 & 89.27 & +13.23  \\
50 & 16.9M & 50.52 & 58.41 & 88.99 & +30.58  \\
60 & 13.5M & 33.57 & 19.72 & 87.59 & +67.87  \\
70 & 10.2M & 18.55 & 10.00 & 85.09 & +75.09  \\
80 & 6.8M  & 14.41 & 10.00 & 72.50 & +62.50   \\
90 & 3.4M  & 10.68 & 10.00 & 28.72 & +18.72   \\
95 & 1.7M  & 10.00 & 10.00 & 13.94 & +3.94   \\
\bottomrule
\end{tabular}
\vspace{5mm}
\caption{\textbf{Pre-finetuning} accuracy on CIFAR-10 of three pruning settings: Baseline Pruning with ReLU, Baseline pruning with TSRA, and CoB Pruning with TSRA. All methods are evaluated across layerwise prune percentages. Post-pruning parameter counts are reported, and $\Delta$ denotes the accuracy difference (TSRA–CoB $-$ TSRA–Base).}
\label{tab:pre_ft-fixed-ratio}
\end{table}

\begin{table}[H]
\centering
\begin{tabular}{c c c c c c}
\toprule
\textbf{Prune \%} &
\textbf{Params After} &
\textbf{ReLU–Base} &
\textbf{TSRA–Base} &
\textbf{TSRA–CoB} &
\textbf{$\Delta$ \begin{tabular}{c}
\end{tabular}} \\
\midrule
0 & 33.6M &  94.49 & 89.97 & 89.97 & +0.00 \\
10 & 30.3M & 93.67 & 89.69 & 89.91 & +0.22 \\
20 & 26.9M & 93.48 & 89.46 & 89.78 & +0.32 \\
30 & 23.7M & 93.29 & 88.98 & 89.64 & +0.66 \\
40 & 20.3M & 93.05 & 88.51 & 89.32 & +0.81 \\
50 & 16.9M & 92.88 & 88.08 & 89.13 & +1.05 \\
60 & 13.5M & 92.61 & 87.07 & 88.85 & +1.78 \\
70 & 10.2M & 92.32 & 85.65 & 88.24 & +2.59 \\
80 & 6.8M  & 91.26 & 82.46 & 85.45 & +2.99 \\
90 & 3.4M  & 89.51 & 71.07 & 79.55 & +8.48 \\
95 & 1.7M  & 84.61 & 50.40 & 64.55 & +14.15\\
\bottomrule
\end{tabular}
\vspace{5mm}
\caption{\textbf{Post-finetuning} accuracy on CIFAR-10 of three pruning settings: Baseline Pruning with ReLU, Baseline pruning with TSRA, and CoB Pruning with TSRA. All methods are evaluated across layerwise prune percentages. Post-pruning parameter counts are reported, and $\Delta$ denotes the accuracy difference (TSRA–CoB $-$ TSRA–Base).}
\label{tab:post_ft-fixed-ratio}
\end{table}

\begin{figure}[H]
    \centering
    \setlength{\fboxsep}{0pt}
    \includegraphics[width=\textwidth]{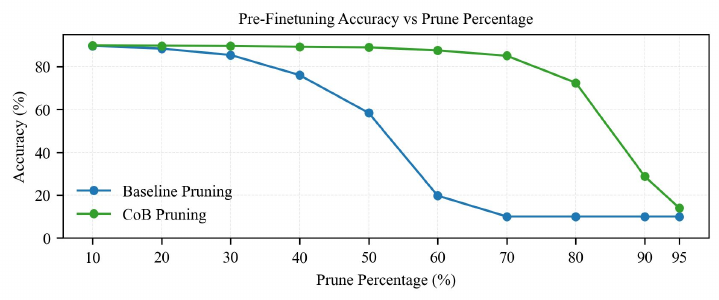}
    \caption{Graph of Pre-finetune accuracy of baseline vs. CoB across percentage of parameters pruned}\label{fig:Pre_FT_Acc_vs_Prune_Ratio}
\end{figure}

\section{Results and Discussion}
\subsection{Results for Layerwise Fixed-Ratio Pruning}

\begin{figure}[H]
    \centering
    \setlength{\fboxsep}{0pt}
    \includegraphics[width=\textwidth]{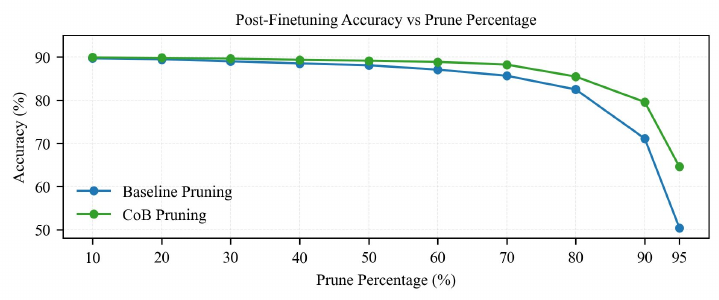}
    \caption{Graph of Post-finetune accuracy of baseline vs. CoB across percentage of parameters pruned}\label{fig:Post_FT_Acc_vs_Prune_Ratio}
\end{figure}

Table~\ref{tab:pre_ft-fixed-ratio} and Table~\ref{tab:post_ft-fixed-ratio} summarize the pre- and post-finetuning performance of baseline pruning versus CoB-augmented pruning across pruning ratios from 10\% to 95\%. Overall, the results reveal that CoB significantly improves robustness to structured pruning, yielding higher accuracy across all prune percentages. Before any finetuning, the baseline model degrades rapidly under layerwise pruning. Accuracy falls below 60\% at 50\% pruned and collapses to chance-level performance (10\%) by 70--80\% pruning. On the other hand, CoB maintains more than 85\% accuracy through 70\% pruning and still retains 72.5\% accuracy at 80\% pruning. This represents accuracy improvements of $+30.6$ to $+75.1$ percentage points relative to the baseline in the 50--80\% pruned regime. This pattern is visualized in Figure~\ref{fig:Pre_FT_Acc_vs_Prune_Ratio}, where the comparison of the degradation curves also indicate that CoB Pruning fundamentally increases the compressibility of the learned representation over the baseline, even before fine-tuning is applied.

After fine-tuning, both methods recover accuracy at lower pruning ratios, but CoB consistently outperforms the baseline across all settings. For moderate pruning (40--60\%), CoB yields improvements ranging from $+0.81$ to
$+1.78$ percentage points. These margins steadily widen at high prune levels. For instance, there exists a clear improvement with CoB over baseline at around $+2.59$ at 70\% pruning and $+2.99$ at 80\% pruning. In extreme compression regimes, the benefits become even more pronounced: at 90--95\% pruning, CoB improves final accuracy by $+8.48$ and $+14.15$ points respectively. The prevalence of this curve is visualized in Figure \ref{fig:Post_FT_Acc_vs_Prune_Ratio}. Together, these results show that CoB expands the feasible compression frontier of the
model. Whereas baseline pruning is only reliable up to roughly 50--60\% pruned, CoB enables accurate inference up to 70--80\% pruning before performance begins to sharply decline. This provides strong empirical evidence that change-of-basis transformations produce representations that are intrinsically more prune-friendly and structurally robust.

\subsection{Results for Layerwise Threshold-Based Pruning}

\begin{table}[H]
\begin{minipage}{0.48\linewidth}
\centering
\subcaptionbox{z-score based Threshold, where threshold T prunes all dimensions with importance z-score < T.\label{tab:zscore}}{
\begin{tabular}{c c c c}
\toprule
Threshold & Prune (\%) & Acc (\%) & $\Delta$ Acc Loss \\
\midrule
-0.75 & 5.89 & 89.56 & -0.44 \\
-0.50 & 23.98 & 89.00 & -1.0 \\
-0.25 & 42.27 & 88.40 & -1.6 \\
0.00  & 96.03 & 84.69 & -5.31 \\
0.25  & 97.28 & 72.16 & -17.84\\
0.50  & 98.10 & 54.32 & -35.68\\
\bottomrule
\end{tabular}
}
\end{minipage}
\hfill
\begin{minipage}{0.48\linewidth}
\centering
\subcaptionbox{Proportion-of-average based threshold; threshold T prunes dimensions with importance below T$\times$(layer mean).\label{tab:propavg}}{
\begin{tabular}{c c c c}
\toprule
Threshold & Prune (\%) & Acc (\%) & $\Delta$ Acc Loss \\
\midrule
0.65 & 93.78 & 87.74 & -2.26 \\
0.75 & 94.60 & 87.07 & -2.93 \\
0.85 & 95.20 & 86.85 & -3.15 \\
0.95 & 95.77 & 85.42 & -4.58 \\
1.00 & 96.01 & 84.47 & -5.53 \\
1.05 & 96.67 & 77.12 & -12.88 \\
\bottomrule
\end{tabular}
}
\end{minipage}

\vspace{1em}

\begin{minipage}{0.48\linewidth}
\centering
\subcaptionbox{Proportion-of-median based threshold; threshold T prunes dimensions with importance below T$\times$(layer median).\label{tab:propmed}}{
\begin{tabular}{c c c c}
\toprule
Threshold & Prune (\%) & Acc (\%) & $\Delta$ Acc Loss\\
\midrule
0.70 & 31.72 & 89.42 & -0.58 \\
0.85 & 52.82 & 88.77 & -1.23 \\
1.00 & 74.96 & 87.30 & -2.7 \\
1.15 & 82.88 & 84.76 & -5.24 \\
1.30 & 86.65 & 81.75 & -8.25 \\
1.45 & 88.86 & 75.93 & -14.07 \\
\bottomrule
\end{tabular}
}
\end{minipage}
\hfill
\begin{minipage}{0.48\linewidth}
\centering
\subcaptionbox{Proportion-of-maximum based threshold; threshold T prunes dimensions with importance below T$\times$(layer max).\label{tab:propmax}}{
\begin{tabular}{c c c c}
\toprule
Threshold & Prune (\%) & Acc (\%) & $\Delta$ Acc Loss \\
\midrule
0.05 & 91.10 & 89.20 & -0.80  \\
0.08 & 93.79 & 88.67 & -1.33 \\
0.11 & 95.33 & 87.57 & -2.43 \\
0.14 & 96.52 & 86.82 & -3.18 \\
0.17 & 97.35 & 83.84 & -6.16 \\
0.20 & 98.04 & 78.80 & -11.20 \\
\bottomrule
\end{tabular}
}
\end{minipage}
\caption{Comparison of CoB pruning performance across four thresholding strategies. Control (no pruning) accuracy was 89.97\%. For each threshold, we report observed prune percentage (fraction of parameters removed), post-finetuning accuracy, and accuracy drop relative to the 89.97\% control.} 

While fixed-ratio pruning provides a controlled benchmark for comparing baseline and CoB pruning, it does not adapt to the activation statistics of each layer. In contrast, threshold-based pruning is ideally positioned to test whether CoB meaningfully restructures activation geometry in ways that make sparsification more efficient. Across all four thresholding strategies: z-score cutoff, proportion-of-average, proportion-of-median, and proportion-of-maximum, we observe a striking and consistent trend: CoB enables the model to withstand aggressive pruning before accuracy collapses. In nearly 3 out of 4 cases, the threshold rules achieve a prune between \textbf{90\% and 96\% of the model’s parameters while maintaining post-finetuning accuracies in the 84--89\% range}, only 1--6 percentage points below the 89.97\% control. This constitutes a dramatic improvement over the layerwise fixed-ratio experiments, where pruning 95\% of parameters reduced CoB accuracy to 64.6\% even after finetuning.\\

For Z-Score thresholding, negative thresholds (e.g., $T=-0.75$ to $-0.25$) prune conservatively and show minimal accuracy loss ($<1.6$ points). At $T=0$, which removes all below-mean dimensions, CoB removes $96\%$ of parameters and still retains $84.7\%$ accuracy. Only for positive thresholds (e.g., $T=0.25$ to $0.50$) do we observe the expected sharp decline. On the other hand, the proportion-of-average strategy behaves more aggressively across layers, producing prune ratios between $94$--$97\%$. Despite this, CoB maintains strong performance, with accuracies between $84$--$88\%$ for thresholds up to $T=1.0$. Performance drops steeply only around $T=1.05$. Median-based thresholding produces a more gradual percent pruned ladder, with pruning increasing smoothly from $32\%$ to $89\%$ as $T$ increases from $0.7$ to $1.45$.  
However, this method yields the lowest overall compression efficiency of the four: accuracy falls below $80\%$ once pruning exceeds $\sim 85\%$. Proportion-of-maximum thresholding delivers the strongest results of all methods. Thresholds in the range $T=0.05$ to $0.14$ prune $91$--$96\%$ of parameters while preserving remarkably high accuracies between $86.8$--$89.2\%$. \textbf{Even at T=0.17 (removing 96.52\% of parameters), the model achieves 86.82\% accuracy, only a 3.18 point drop from the unpruned baseline.} This experiment provides empirical evidence that CoB supports extremely aggressive pruning without proportional accuracy loss.\\  

\label{tab:all_cob_pruning}
\end{table}
\begin{figure}[H]
    \centering
    \setlength{\fboxsep}{0pt}
    \includegraphics[width=\textwidth]{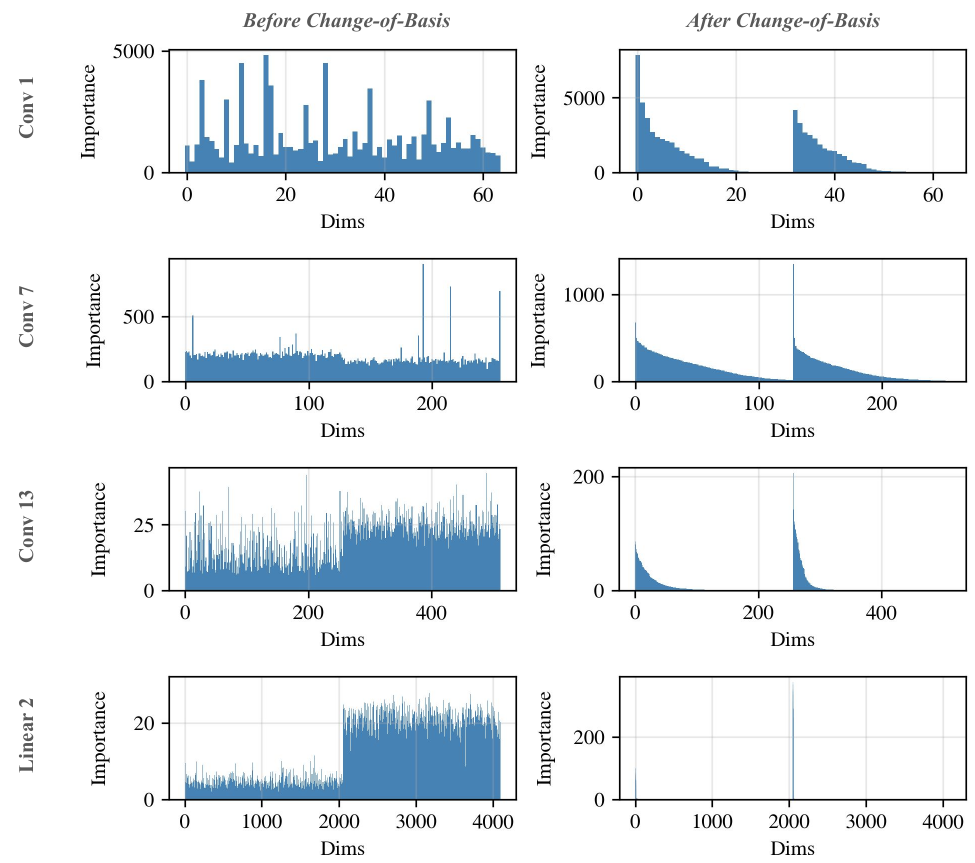}
    \caption{Importance scores for certain layers of VGG-16 before and after change-of-basis to concentrate importance. The two-subspace structure of TSRAs is evident in the separation of importance scores between the first half and second half of component dimensions. As the size of the hidden layers increase, there appears to be a trend of increasing kurtosis in the distribution of post-change-of-basis importance scores, indicating the model is able to concentrate importance into a very small number of the many dimensions.}
    \label{fig:pre_post_cob_importance_scores}
\end{figure}

As shown in Figure \ref{fig:pre_post_cob_importance_scores}, after the change-of-basis transformation, each layer’s importance scores become highly concentrated in a small number of dimensions, in contrast to the diffuse, irregular distributions seen in the baseline model. This concentration effect enables threshold-based pruning to operate so effectively. Collectively, these results indicate that CoB improves pruning robustness and reshapes the model's activation distribution, enabling higher pruning ratios with the tested methods.

\section{Conclusion and Future Work}

In this work, we study change-of-basis (CoB) pruning as a principled approach to structured model compression, enabled by our introduction of Two-Subspace Radial Activations (TSRAs) that admit rotational invariance within each of their subspaces. We provide a proof showing that TSRAs are not subject to the same width saturation limitations as the more generally rotationally invariant radial rescaling activations and prove an upper bound on their width saturation capabilities. This invariance of TSRAs and additional modifications we make allows orthogonal linear transformations to be cleanly merged into neighboring weights, enabling activation-aligned representations that are more compressible than those originally learned. Our experiments on VGG-16 demonstrate that CoB consistently improves pruning robustness across both fixed-ratio and threshold-based schedules, extending the reliable pruning frontier from $\sim$ 30\% to nearly 70\% of parameters removed on pre-finetuning layerwise fixed pruning experiments, and allowing 90–96\% compression under thresholds with only modest accuracy loss. Our results serve as a promising proof-of-concept into change-of-basis pruning enabled by rotational invariance. While TSRAs provide the necessary geometric structure for CoB, our analysis only demonstrates an upper bound on their width saturation properties. Future work may provide a more complete mathematical characterization, demonstrating lower bounds under certain conditions or tightening the upper bound. Additionally, future work may also focus on exploring other activation functions within the TSRA family, investigating other activation families with rotational invariance properties, developing weight initialization schemes for any of these activation functions, or exploring alternative importance metrics or pruning schemes (eg. unstructured pruning). Together, these directions may further expand the viability of change-of-basis pruning methods enabled by rotational invariance.

\section{Acknowledgments} 

We sincerely thank Professor Yen-Ling Kuo at the University of Virginia for supporting and advising our research endeavors in general and for providing access to the computational resources used for this work. Furthermore, we thank the University of Virginia Research Computing and the Department of Computer Science for providing the high-performance computing systems that made this research possible.

\bibliographystyle{unsrt}  
\bibliography{references}

\appendix

\section{Width Saturation of Two-Subspace Rescaling Activations}
\label{sec:tsra_width_saturation_proof}

We characterize the width saturation properties of TSRAs by showing that they are generally strictly capable of greater width utilization compared to radial rescaling activations and provide an upper bound on their width saturation.

\subsection{TSRAs Do Not Preserve All Linear Subspaces}

Let $\R^d = U \oplus V$ be an orthogonal decomposition where $U$ and $V$ are non-zero subspaces as well as $x_U \in U$ and $x_V \in V$ be the decomposition of $x\in \R^d$. Denote functions $f_U, f_V: \R^2 \to \R$ where there exists scalars $a, b > 0$ such that $f_U(a,b) \neq f_V(a,b)$ and the corresponding TSRA $\sigma(x) = f_U(|x_U|,|x_V|)x_U + f_V(|x_U|,|x_V|)x_V$. Additionally, let $\phi(x)$ be a radial rescaling function.

As demonstrated in section~\ref{sec:tsras}, for any vector $x$, $\phi(x) \in \spa(\{x\})$. That is, $\phi(x)$ will still sit on the line spanned by $x$, since the radial rescaling function can only rescale the vector but not change its direction. As a result, given a subspace $S \subset \R^d$ where $\dim(S) < d$, $\phi(S) \subseteq S$ which implies that $\dim(\phi(S)) \leq \dim(S) < d$. Thus, the radial rescaling function $\phi$ is \textit{not} capable of utilizing any of the additional dimensionality provided by the higher-dimensional space $\R^d$.

\begin{lemma}
    Given $\sigma$ as defined above, there exists a $1$-dimensional subspace $L \subset \R^d$ such that $\sigma(L) \not\subseteq L$. As a result, TSRAs do not preserve all linear subspaces.
\end{lemma}

\begin{proof}

Using the definitions above, set $x$ such that $x_U, x_V \neq 0$ and $f_U(|x_U|, |x_V|) \neq f_V(|x_U|, |x_V|)$ which is possible given the definitions of $f_U$ and $f_V$. Then

$$\sigma(x) = f_U(|x_U|,|x_V|)x_U + f_V(|x_U|,|x_V|)x_V$$

If $L = \spa (\{x\})$ were invariant under $\sigma$, then we would have $\sigma(x) \in L$, which implies that there exists some scalar $\lambda$ such that

$$\sigma(x) = f_U(|x_U|,|x_V|)x_U + f_V(|x_U|,|x_V|)x_V = \lambda x = \lambda(x_U + x_V) = \lambda x_U + \lambda x_V$$

Abbreviating $f_U = f_U(|x_U|,|x_V|)$ and $f_V = f_V(|x_U|,|x_V|)$, we have:

$$f_Ux_U + f_Vx_V = \lambda x_U + \lambda x_V$$

Rearranging

$$(f_U-\lambda)x_U + (f_V - \lambda)x_V = 0$$

However, since $x_U \in U$, $x_V \in V$, and $U \perp V$ with $x_U, x_V \neq 0$, this implies that $f_U = \lambda$ and $f_V = \lambda$ and thus $f_U = f_V$. But this contradicts $f_U \neq f_V$.

As a result, we conclude that $L$ is not invariant under $\sigma$, and thus that TSRAs do not preserve all linear subspaces.

\end{proof}

By demonstrating that TSRAs do not preserve all linear subspaces, we show that it has strictly \textit{greater} width saturation than radial rescaling activations. Next, we prove an upper bound for the width saturation of TSRAs.

\subsection{TSRA Width Saturation Upper Bound}

Keep the previous definitions for $\sigma$, $f_U$, and $f_V$. Let $d_i$ be the size of the $i$-th hidden layer and $d_{i-1}$ the size of the hidden layer before it. Let $W \in \R^{d_i \times d_{i-1}}$ be the weight of the $i$-th hidden layer and $b \in \R^{d_i}$ its optional bias. Note that for the purpose of conciseness, in all following notation the addition of $b$ will be used alongside the weight, \textit{however}, it should be interpreted as \textit{optional}. Denote the orthogonal decomposition $\R^{d_i} = U \oplus V$. Let $x_\text{in} \in \R^{d_{i-1}}$ represent the input vector to the hidden layer, $x_\text{pre} = Wx_\text{in} + b$ be the pre-activation, and $x_\text{act} = \sigma(x_\text{pre})$ the activation. Also denote $I$ the set of all possible $x_\text{in}$, $P$ the set of all possible $x_\text{pre}$, and $S$ as the \textit{span} of all possible $x_\text{act}$. Ie. $P = WI + b$ and $S = \spa(\sigma(P))$.

\begin{lemma}\label{lemma:tsra_upper_bound}
    Using the above definitions the upper bound for $\dim(S)$ is $2(d_{i-1} + 1)$ in the case of a bias and $2d_{i-1}$ in the case of no bias.
\end{lemma}

\begin{proof}
    Denote $Q_U$ as the orthogonal projection from $\R^{d_i}$ to $U$ such that $x_U = Q_Ux$. Similarly, denote $Q_V$ as the orthogonal projection from $\R^{d_i}$ to $V$ such that $x_V = Q_Vx$. Abbreviate $f_U = f_U(|x_{\text{pre,}U}|,|x_{\text{pre,}V}|)$ and $f_V = f_V(|x_{\text{pre,}U}|,|x_{\text{pre,}V}|)$. Then, we can obtain the activations via the equation:

    $$x_\text{act} = f_UQ_Ux_\text{pre} + f_VQ_Vx_\text{pre}$$

    This shows that $x_\text{act}$ always lies on the two-dimensional plane spanned by $Q_Ux_\text{pre}$ and $Q_Vx_\text{pre}$. As a result, the span of all activations, $S$, must belong within the combined span of all possible $Q_Ux_\text{pre}$ and $Q_Vx_\text{pre}$:

    $$S \subseteq \spa(Q_UP \cup Q_VP)$$

    Since $Q_UP \subseteq U$ and $Q_VP \subseteq V$, and $U \perp V$, this span is just the direct sum:

    $$S \subseteq \spa(Q_UP \cup Q_VP) = \spa(Q_UP) \oplus \spa(Q_VP)$$
    
    Which implies

    $$\dim(S) \leq \dim(\spa(Q_UP)) + \dim(\spa(Q_VP))$$

    $P = WI + b$ and $\dim(\spa(I)) \leq d_{i-1}$ by definition. As a result, $\dim(\spa(Q_UP)), \dim(\spa(Q_VP)) \leq d_{i-1}$ in the case no bias is used. If a bias is used, $\spa(Q_UP) = \spa(Q_UWI \cup \{Q_Ub\})$ and $\spa(Q_VP) = \spa(Q_VWI \cup \{Q_Vb\})$ and thus $\dim(\spa(Q_UP)), \dim(\spa(Q_VP)) \leq d_{i-1} + 1$. As a result, in the case of bias we can conclude:

    $$\dim(S) \leq 2(d_{i-1} + 1)$$

    And in the case of no bias:

    $$\dim(S) \leq 2d_{i-1}$$
\end{proof}

Through Lemma~\ref{lemma:tsra_upper_bound}, we have shown that TSRAs have an upper bound for maximum width saturation at double the dimensionality of the previous hidden layer's size, plus $2$ if a bias is used. As a result, TSRAs may be capable of exponential increase of width saturation through many hidden layers, which would enable practical utilization of high-width hidden layers in many deep architectures.

\end{document}